\documentclass[a4paper, 11pt]{article}
\usepackage{fullpage}

\usepackage{hyperref} 
\usepackage{xcolor}
\hypersetup{
    colorlinks,
    linkcolor={red!50!black},
    citecolor={blue!50!black},
    urlcolor={blue!80!black}
}

\usepackage{mathtools} 
\usepackage{soul} 
\usepackage{bm} 
\usepackage{amsfonts} 
\usepackage[T1]{fontenc} 
\usepackage{authblk} 

\usepackage{natbib}
\setcitestyle{authoryear,open=(,close=)} 

\usepackage{amsthm} 
\newtheorem{theorem}{Theorem}
\newtheorem{definition}{Definition}

\newtheorem{example}{Example}

\usepackage{tikz}
\usetikzlibrary{arrows.meta}
\usetikzlibrary{calc}
\usetikzlibrary{fit}
\usetikzlibrary{positioning}
\usetikzlibrary{external}
\renewcommand{\L}{\mathcal{L}}
\renewcommand{\H}{\mathcal{H}}
\newcommand{\M}{\mathcal{M}}
\newcommand{\Mi}{{\mathcal{M}_i}}

\newcommand{\Ltuple}{(\bm{X},\allowbreak\bm{E},\allowbreak\bm{F},\allowbreak\Pr_{\bm{E}})}
\newcommand{\Htuple}{(\bm{Y},\allowbreak\bm{U},\allowbreak\bm{G},\allowbreak\Pr_{\bm{U}})}
\newcommand{\Mtuple}{(\bm{X},\allowbreak\bm{E},\allowbreak\bm{F},\allowbreak\Pr_{\bm{E}})}

\newcommand{\Mituple}{(\bm{X},\allowbreak\bm{E},\allowbreak\bm{F}^i,\allowbreak\Pr_{\bm{E}})}

\newcommand{\I}{\mathcal{I}}
\newcommand{\J}{\mathcal{J}}

\renewcommand{\Pr}{\mathbb{P}}
\newcommand{\emp}{\varepsilon}

\DeclareMathOperator{\even}{Even}
\DeclareMathOperator{\proj}{Proj}
\DeclareMathOperator{\val}{Val}
\DeclareMathOperator{\img}{Image}
\DeclareMathOperator{\rst}{Rst}
\DeclareMathOperator{\softrst}{SoftRst}

\DeclareMathOperator{\pa}{Pa}

\title{Causal Abstraction with Soft Interventions}
\author[1]{Riccardo Massidda}
\author[2]{Atticus Geiger}
\author[2]{Thomas Icard}
\author[1]{Davide Bacciu}
\affil[1]{Università di Pisa}
\affil[ ]{\texttt{\small riccardo.massidda@phd.unipi.it, davide.bacciu@unipi.it}}
\affil[2]{Stanford University}
\affil[ ]{\texttt{\small \{atticusg,icard\}@stanford.edu}}
\date{\vspace{-5ex}}

\begin{document}
\maketitle

\begin{abstract}
Causal abstraction provides a theory describing how several causal models can represent the same system at different levels of detail. Existing theoretical proposals limit the analysis of abstract models to ``hard'' interventions fixing causal variables to be constant values. In this work, we extend causal abstraction to ``soft'' interventions, which assign possibly non-constant functions to variables without adding new causal connections.
Specifically, (i) we generalize $\tau$-abstraction from \citet{beckers2019abstracting} to soft interventions, (ii) we propose a further definition of soft abstraction to ensure a unique map~$\omega$ between soft interventions, and (iii) we prove that our constructive definition of soft abstraction guarantees the intervention map~$\omega$ has a specific and necessary explicit form.
\end{abstract}

\section{Introduction}

Causal modeling is a crucial tool to represent, reason, and act on systems composed of causally connected independent mechanisms~\citep{peters2017elements,bareinboim:etal20}. A causal analysis requires a choice of variables; that is, we must choose the level of \emph{abstraction} in which to couch an analysis \citep{woodwardchoice}.  
The question becomes especially salient with a large number of densely connected variables that are difficult to analyze or understand.
As such, the fundamental concept of \textit{causal abstraction} has played a key role in the scientific investigation of complex phenomena, including weather patterns \citep{chalupka2016}, brains \citep{Dubois2020PersonalityBT,Dubois2020CausalMO}, and deep learning artificial intelligence models \citep{Chalupka:2015, geiger-etal:2020:blackbox, geiger2021causal}.

Within causal modeling, there is a recent growing body of literature on formal theories of causal abstraction \citep{Chalupka,rubenstein2017causal, zennaro2022abstraction}. However, following the widespread $\operatorname{do}$-operator~\citep{pearl2009causality}, existing theories of causal abstraction only consider \textit{hard interventions}, where a variable is fixed to a single value. Not all mechanistic changes to a real-world system can be represented by a hard intervention. For instance, increasing network weights in a deep learning model is not a matter of fixing the values in representations; the causal mechanisms have been altered, but not fixed. In this context, soft interventions formalize the replacement of causal mechanisms with different and possibly non-constant mechanisms~\citep{eberhardt2007interventions}. 

In this paper, we extend previous work on $\tau$-abstraction~\citep{beckers2019abstracting} to address the description of non-constant abstracted mechanisms.
In Section~\ref{sec:softabstraction}, we introduce a generalization of the underlying formal concepts and we prove that, while correctly handling low-level soft interventions, the resulting definition of $\tau$-abstraction is ambiguous for high-level soft interventions.
Consequently, we propose a novel definition of soft abstraction by strengthening the consistency requirement of $\tau$-abstraction on endogenous settings. Notably, we are able to prove that our extension maintains desirable properties of $\tau$-abstraction and that it is equivalent to the latter whenever the two models are hard-intervened.
Finally, in Section~\ref{sec:constructivesoftabstraction}, we specialize our definition to the constructive case, where each high-level variable depends on a subset of the low-level variables, as in Figure~\ref{fig:explicitomega}. As a further contribution, we prove that constructive soft abstraction implies a unique explicit form on the function mapping interventions from the lower- to the higher-level model.

\definecolor{qcolor}{RGB}{127, 161, 170}
\definecolor{pcolor}{RGB}{222, 181, 106}
\definecolor{icolor}{RGB}{180, 40, 40}
\definecolor{ycolor}{RGB}{74, 154, 71}

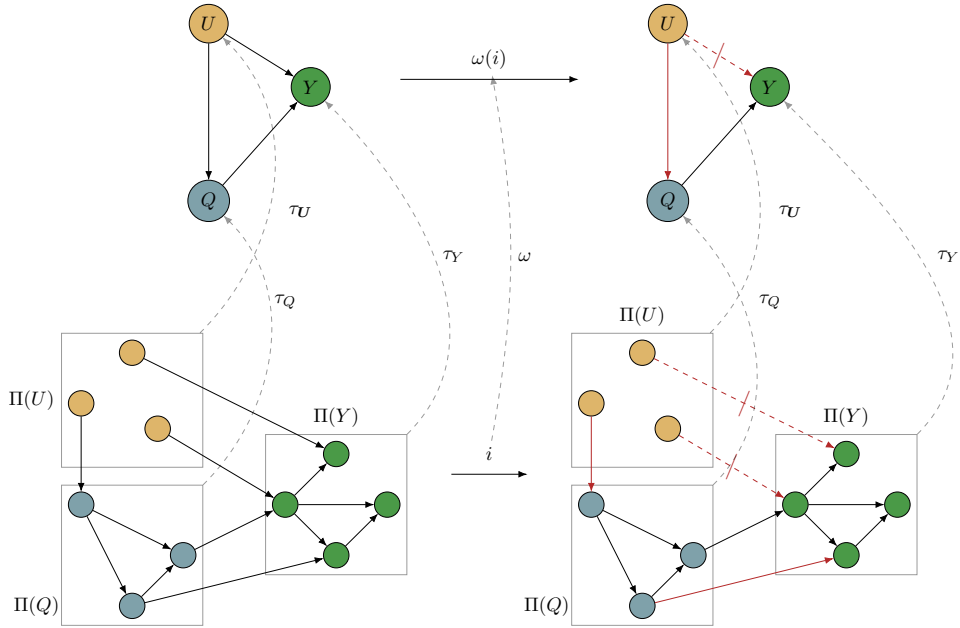
\begin{figure}[t]
\centering
\resizebox{0.85\textwidth}{!}{
\begin{tikzpicture}[minimum width=0.5cm]
{
  \def\opac{0.4}
  \def\xshift{-3.5}
  \def\yshift{7}
  \node[draw, circle,fill=pcolor] (P) at (-5 + \xshift, -0.5+\yshift) {$U$};
  \node[draw, circle,fill=qcolor] (Q) at (-5 + \xshift, -4+\yshift) {$Q$};
  \node[draw, circle,fill=ycolor] (H) at (-3 + \xshift, -1.75+\yshift) {$Y$};
  \draw[-{Latex}] (P)--(Q);
  \draw[-{Latex}] (P)--(H);
  \draw[-{Latex}] (Q)--(H);
  \node[draw=none,rectangle, opacity = \opac, fit=(P) (Q) (H)] (PQH) {};
  
  \def\xshift{5.5}
  \def\yshift{7}
  \node[draw, circle,fill=pcolor] (P2) at (-5 + \xshift, -0.5+\yshift) {$U$};
  \node[draw, circle,fill=qcolor] (Q2) at (-5 + \xshift, -4+\yshift) {$Q$};
  \node[draw, circle,fill=ycolor] (H2) at (-3 + \xshift, -1.75+\yshift) {$Y$};
  \draw[-{Latex},color=icolor] (P2)--(Q2);
  \draw[-{Latex},color=icolor,dashed] (P2)--(H2);
  \node[draw=none](STRIKE) at ($(P2)!0.5!(H2)$){\color{icolor}\Large\textbf{/}};
  \draw[-{Latex}] (Q2)--(H2);
  \node[draw=none,rectangle, opacity = \opac, fit=(P2) (Q2) (H2)] (PQH2) {};

  \node[draw, circle,fill=pcolor] (LP1) at (0, 0) {};
  \node[draw, circle,fill=pcolor] (LP2) at (-1, -1) {};
  \node[draw, circle,fill=pcolor] (LP3) at (0.5, -1.5) {};
  \node[draw=none] (BOUNDARYRIGHT) at (1, -2) {};
  \node[draw,rectangle, opacity = \opac, fit=(LP1) (LP2) (BOUNDARYRIGHT),label={$\Pi(U)$}] (LPlab) {};
  \node[draw, circle,fill=qcolor] (LQ1) at (-1, -3) {};
  \node[draw, circle,fill=qcolor] (LQ2) at (1, -4) {};
  \node[draw, circle,fill=qcolor] (LQ3) at (0, -5) {};
  \node[draw,rectangle, opacity = \opac, fit=(LQ1) (LQ2) (LQ3),label={[yshift=-1cm,xshift=0.1cm]left:$\Pi(Q)$}] (LQlab) {};
  \node[draw, circle,fill=ycolor] (LH1) at (3, -3) {};
  \node[draw, circle,fill=ycolor] (LH2) at (4, -2) {};
  \node[draw, circle,fill=ycolor] (LH3) at (4, -4) {};
  \node[draw, circle,fill=ycolor] (LH4) at (5, -3) {};
  \node[draw,rectangle, opacity = \opac, fit=(LH1) (LH2) (LH3) (LH4),label={$\Pi(Y)$}] (LHlab) {};
  \draw[-{Latex},color=icolor,dashed] (LP1)--(LH2);
  \draw[-{Latex},color=icolor,dashed] (LP3)--(LH1);
  \draw[-{Latex},color=icolor] (LP2)--(LQ1);
  \draw[-{Latex}] (LQ1)--(LQ2);
  \draw[-{Latex}] (LQ1)--(LQ3);
  \draw[-{Latex}] (LQ2)--(LH1);
  \draw[-{Latex}] (LQ3)--(LQ2);
  \draw[-{Latex},color=icolor] (LQ3)--(LH3);
  \draw[-{Latex}] (LH1)--(LH2);
  \draw[-{Latex}] (LH1)--(LH3);
  \draw[-{Latex}] (LH1)--(LH4);
  \draw[-{Latex}] (LH3)--(LH4);
  
 \def\xshift{-10}

  \node[draw, circle,fill=pcolor] (LP12) at (0 + \xshift, 0) {};
  \node[draw, circle,fill=pcolor] (LP22) at (-1+ \xshift, -1) {};
  \node[draw, circle,fill=pcolor] (LP32) at (0.5+ \xshift, -1.5) {};
  \node[draw=none] (BOUNDARYLEFT) at (1+ \xshift, -2) {};
  \node[draw,rectangle, opacity = \opac, fit=(LP12) (LP22) (BOUNDARYLEFT),label={left:$\Pi(U)$}] (LPlab2) {};
  \node[draw, circle,fill=qcolor] (LQ12) at (-1+ \xshift, -3) {};
  \node[draw, circle,fill=qcolor] (LQ22) at (1+ \xshift, -4) {};
  \node[draw, circle,fill=qcolor] (LQ32) at (0+ \xshift, -5) {};
  \node[draw,rectangle, opacity = \opac, fit=(LQ12) (LQ22) (LQ32),label={[yshift=-1cm,xshift=0.1cm]left:$\Pi(Q)$}] (LQlab2) {};
  \node[draw, circle,fill=ycolor] (LH12) at (3+ \xshift, -3) {};
  \node[draw, circle,fill=ycolor] (LH22) at (4+ \xshift, -2) {};
  \node[draw, circle,fill=ycolor] (LH32) at (4+ \xshift, -4) {};
  \node[draw, circle,fill=ycolor] (LH42) at (5+ \xshift, -3) {};
  \node[draw,rectangle, opacity = \opac, fit=(LH12) (LH22) (LH32) (LH42),label={$\Pi(Y)$}] (LHlab2) {};
  \draw[-{Latex}] (LP12)--(LH22);
  \draw[-{Latex}] (LP32)--(LH12);
  \draw[-{Latex}] (LP22)--(LQ12);
  \draw[-{Latex}] (LQ12)--(LQ22);
  \draw[-{Latex}] (LQ12)--(LQ32);
  \draw[-{Latex}] (LQ22)--(LH12);
  \draw[-{Latex}] (LQ32)--(LQ22);
  \draw[-{Latex}] (LQ32)--(LH32);
  \draw[-{Latex}] (LH12)--(LH22);
  \draw[-{Latex}] (LH12)--(LH32);
  \draw[-{Latex}] (LH12)--(LH42);
  \draw[-{Latex}] (LH32)--(LH42);


  
  \node[draw=none](STRIKE) at ($(LP1)!0.5!(LH2)$){\color{icolor}\Large\textbf{/}};
  \node[draw=none](STRIKE) at ($(LP3)!0.5!(LH1)$){\color{icolor}\Large\textbf{/}};
  
  \node[] (bigA) at (-5,5.4) {};
  \node[] (bigB) at (-1,5.4) {};
  \node[](omegai) at (-3,5.8) {$\omega(i)$};
  \draw[-{Latex}] (bigA) -- (bigB);
  
  \node[] (bigA) at (-4,-2.4) {};
  \node[] (bigB) at (-2,-2.4) {};
  \node[](i) at (-3,-2) {$i$};
  \draw[-{Latex}] (bigA) -- (bigB);
  
 \draw[-{Latex},dashed,opacity=\opac] (LPlab2) to[in = -45, out = 45]  (P);
 \draw[-{Latex},dashed,opacity=\opac] (LHlab2) to[in = -45, out = 45]  (H);
 \draw[-{Latex},dashed,opacity=\opac] (LQlab2) to[in = -45, out = 45]  (Q);
 
 \draw[-{Latex},dashed,opacity=\opac] (LPlab) to[in = -45, out = 45]  (P2);
 \draw[-{Latex},dashed,opacity=\opac] (LHlab) to[in = -45, out = 45]  (H2);
 \draw[-{Latex},dashed,opacity=\opac] (LQlab) to[in = -45, out = 45]  (Q2);
 
 \draw[-{Latex},dashed,opacity=\opac] (i) to[in = -80, out = 80]  (omegai);
 
  \node[](omegai) at (-2.3,1.9) {$\omega$};
 
  \node[](omegai) at (-6.7,2.8) {$\tau_{\bm{U}}$};
  \node[](omegai) at (-7,1) {$\tau_Q$};
  \node[](omegai) at (-3.7,1.9) {$\tau_Y$};
 
  \node[](omegai) at (2.9,2.8) {$\tau_{\bm{U}}$};
  \node[](omegai) at (2.5,1) {$\tau_Q$};
  \node[](omegai) at (6,1.9) {$\tau_Y$};
  
}
\end{tikzpicture}
}
\caption{
A visual depiction of constructive soft abstraction. The alignment between the high and low-level models is given by a map $\Pi$ from high-level variables to sets of low-level variables and functions $\tau$ mapping exogenous and endogenous settings. The surjective partial function~$\omega$ maps a low-level intervention $i$ to a high-level intervention~$\omega(i)$. Notably, the function~$\omega$ is uniquely induced by the functions $\tau$, and we provide its explicit form. Interventions may be hard (indicated by the red slash through the dashed red arrow) or soft (indicated by a solid red arrow with no slash). Observe that an intervention being performed on the connection between low-level clusters does not necessarily result in a high-level interventions.
}\label{fig:explicitomega}
\end{figure}

\section{Related Work}\label{sec:related}

\citet{rubenstein2017causal} introduce the notion of an ``exact transformation'' to understand when a probabilistic causal model can be transformed into another model in a causally consistent way. The core mathematical objects in an exact transformation are $\tau$, a map between total settings of the causal models, and $\omega$, a surjective and order-preserving map between hard interventions. \citet{beckers2019abstracting} build on this work, introducing the notion of $\tau$-abstraction, which requires $\tau$ to induce a specific function~$\omega$ in order to rule out unintuitive aspects of exact transformation. Our work directly extends their definitions to handle soft interventions. Since our work focuses on redefining foundations of causal abstraction for soft interventions, we do not tackle the problem of measuring approximate abstraction~\citep{beckers20a,rischel2021compositional}. For an in depth survey of similar abstraction relations, we refer the reader to \citet{zennaro2022abstraction}, which reviews various causal abstraction proposals for hard intervened models.

\section{Structural Causal Models}\label{sec:structuralcausalmodels}

Throughout the paper, we use notation $\bm{X}$ to denote a fixed set of variables, each ${V \in \bm{X}}$ with domain $\val(V)$ of possible values. We refer to set of all possible subsets of variables as the power~set $\mathcal{P}(\bm{X})$. The domain of any subset of variables $\bm{V} \subseteq \bm{X}$ is the Cartesian product of the members domains $\val(\bm{V}) = \bigtimes_{V \in \bm{V}}\val(V)$.
We refer to the values $\bm{v} \in \val(\bm{V})$ as \emph{partial settings}.
Similarly, we call $\bm{x} \in \val(\bm{X})$ a \emph{total setting}.
Finally, given a total setting $\bm{x}\in\val(\bm{X})$ and a set of variables $\bm{V} \subseteq \bm{X}$, we define the projection $\proj(\bm{x},\bm{V})\in\val(\bm{V})$ as the restriction of $\bm{x}$ to the variables~$\bm{V}$.

Structural Causal Models (SCMs) express causal connections as functional relations between variables~\citep{bongers2021foundations}.
In particular, SCMs distinguish between endogenous variables, for which the causal mechanism is determined, and exogenous independent random variables.
\begin{definition}[Structural Causal Model]
    A Structural Causal Model~$\M$ is a tuple $\Mtuple$ such that
    \begin{enumerate}
        \item $\bm{X}$ is a set of endogenous variables with domain $\val(X)$ for each $X\in\bm{X}$,
        \item $\bm{E}$ is a set of exogenous variables with domain $\val(E)$ for each $E\in\bm{E}$,
        \item $\bm{F}$ is a set of functions with form $F_X \colon \val(\bm{X}) \times \val(\bm{E}) \to \val(X)$ for each $X\in\bm{X}$,
        \item $\Pr_E$ is a probability measure on $\val(\bm{E})$.
    \end{enumerate}
\end{definition}
For each endogenous variable ${X\in\bm{X}}$, the structural equation $F_X$ depends only on a subset of variables.
We refer to this subset as the parents $\pa(X) \subseteq{\bm{X}\cup\bm{E}}$ of the variable $X$.
Endogenous settings are fully determined by exogenous settings.
Therefore, each model can be considered as a function $\mathcal{M} \colon \val(\bm{E}) \to \val(\bm{X})$ from exogenous to endogenous values.

An intervention describes an operation on an SCM that produces a possibly different causal model by modifying its structural equations.
We refer to interventions fixing structural equations to constant values as ``hard'' interventions~\citep{pearl2009causality}.
\begin{definition}[Hard Intervention]
Given an SCM $\M = (\bm{X}, \bm{E}, \bm{F}, \Pr_E)$, a subset of endogenous variables ${\bm{V}\subseteq\bm{X}}$, and a partial setting ${\bm{v}\in\val(\bm{V})}$, a hard intervention $i = (\bm{V} \gets \bm{v})$ produces an SCM $\M_i = (\bm{X}, \bm{U}, \bm{F}^i, \Pr_E )$ such that
\begin{equation}
  F_X^i(\bm{x}, \bm{e}) = 
  \begin{cases}
    \proj(\bm{v},X) & X \in \bm{V}\\
    F_X(\bm{x}, \bm{e}) & \text{otherwise}
  \end{cases}
\end{equation}
for each endogenous variable $X\in\bm{X}$,
\end{definition}

By constraining one or more endogenous variables, a hard intervention reduces the set of values possibly assumed by an SCM\@.
Following~\citet{beckers2019abstracting}, given an intervention on a model, we refer to the resulting subset of assumable values as its restriction set.
\begin{definition}[Restriction Set]
The restriction $\rst(\Mi)$ of a hard intervention $i={(\bm{V}\gets \bm{v})}$ on a model $\M$ is the subset of total settings on $\bm{X}\supseteq\bm{V}$ matching the partial setting $\bm{v}$. Formally,
    \begin{equation}
        \rst(\M_i) = \{{\bm{x}\in\val(\bm{X})}\mid{\bm{v}=\proj(\bm{x},\bm{V})}\}.
    \end{equation}
\end{definition}

\section{Causal Abstraction with Hard Interventions}\label{sec:hardabstraction}

Intuitively, \citet{beckers2019abstracting} define an abstraction between two causal models whenever a function on their variables settings induces a function between interventions.
The intervention map must be consistent: intervening on the higher-level abstraction should be equivalent to abstracting the intervened lower-level model.
\begin{definition}[Abstraction]\label{def:abstraction}
Let $\L=\Ltuple$ and $\H=\Htuple$ be two SCMs, $\I$ and $\J$ be their respective sets of admissible hard interventions, $\tau_{\bm{Y}} \colon \val(\bm{X}) \to \val(\bm{Y})$ be a surjective function between endogenous settings, and $\tau_{\bm{U}} \colon \val(\bm{E}) \to \val(\bm{U})$ be a surjective function between exogenous settings. We can use $\tau_{\bm{Y}}$ to induce $\omega\colon\I\to\J$, a partial function from low-level hard interventions to high-level hard interventions. Specifically, $\omega$ is defined on the low-level intervention $i\in\I$ whenever there exists an high-level intervention $\omega(i)\in\J$ such that 
\begin{equation}\label{eq:omega}
    \rst(\H_{\omega(i)})= \{\tau_{\bm{Y}}(\bm{x}) \mid \bm{x} \in \rst(\L_i)\}.
\end{equation}
The model $\H$ is a $\tau$-abstraction of $\L$ whenever $\omega$ is surjective and, for any intervention~$i \in\I$ and exogenous setting~$\bm{e}\in\val(\bm{E})$ at the lower level, it holds
\begin{equation}\label{eq:compatibility}
    \tau_{\bm{Y}}(\L_i(\bm{e})) = \H_{\omega(i)}(\tau_{\bm{U}}(\bm{e})).
\end{equation}
\end{definition}

Given Definition~\ref{def:abstraction}, whenever the function $\omega$ between interventions exists, it is unique. Furthermore, by ensuring consistency on all exogenous settings, abstraction entails that the high-level model is an exact transformation \citep{rubenstein2017causal} between SCMs for every possible exogenous distribution. Therefore, whenever a model abstracts another, the intervention mapping~$\omega$ has a fixed point in the empty intervention and preserves the following intervention ordering.
\begin{definition}[Hard Interventions Ordering]
    A hard intervention $i_1 = {(\bm{V} \gets \bm{v})}$ on an SCM~$\M$ precedes another intervention $i_2 = {(\bm{W} \gets \bm{w})}$ whenever the former sets to the same values a subset of the variables intervened by the latter. Formally,
    \begin{equation}
        i_1 \sqsubseteq i_2 \iff \bm{V}\subseteq\bm{W} \land \bm{v} = \proj(\bm{w}, \bm{V}).
    \end{equation}
\end{definition}

\section{Causal Abstraction with Soft Interventions}\label{sec:softabstraction}

In this section, we extend $\tau$-abstraction to relate causal models according to their response to non-constant interventions on their endogenous variables. Firstly, we generalize the necessary definitions of restriction set and partial ordering between hard interventions (Subsection~\ref{subsec:softinterventions}). Then, we define low soft abstraction to handle soft interventions on the low-level model (Subsection~\ref{subsec:lowsoftabstraction}). In this context, we introduce a running example showing how our definition correctly characterizes low-level soft abstracted models. Then, we prove that this preliminary definition is ambiguous for soft interventions on the higher-level model. Therefore, we modify our example to display an ambiguous scenario. Finally, we present our definition of soft abstraction, which correctly disambiguates higher-level soft interventions (Subsection~\ref{subsec:softabstraction}). Consequently, we showcase how our contribution correctly solves the critical scenario in our running example.

\subsection{Soft Interventions}\label{subsec:softinterventions}

Existing definitions of abstraction explicitly build on ``hard'' interventions that fix the structural equation of each intervened variable to a constant.
A \textit{soft} intervention is a generalization of an hard intervention that instead replaces structural equations with new, possibly non-constant functions.
\begin{definition}[Soft Intervention]
Given an SCM $\M = \Mtuple$,
a subset of endogenous variables ${\bm{V}\subseteq\bm{X}}$,
and a set of functions $\bm{f}$,
a soft intervention $i = {(\bm{V} \gets \bm{f})}$ produces
an SCM $\Mi = \Mituple$ such that
\begin{equation}
  F_X^i(\bm{x}, \bm{e}) = 
  \begin{cases}
    f_X(\bm{x}, \bm{e}) & X \in \bm{V}\\
    F_X(\bm{x}, \bm{e}) & \text{otherwise}
  \end{cases}
\end{equation}
for each endogenous variable $X\in\bm{X}$. 
\end{definition}

For our purposes, we will impose two constraints on soft interventions.
First, for each variable ${X\in\bm{V}}$ the function $f_X$ must share domain and codomain with the non-intervened structural equation $F_X$.
In this way, we enforce that a soft intervention does not alter the ``type'' of a variable.
Second, we require intervened structural equations to depend at most on the same parents $\pa(X)$ from the non-intervened model, meaning no new causal dependencies are formed between variables.
Given this formalization, hard interventions are a special case of soft interventions where structural equations are replaced by constant functions.

To define causal abstraction between soft-intervened models, we must first extend the intervention map~$\omega$ to be a partial function between soft interventions.
Abstraction uniquely induces such function given the restrictions of hard interventions on the lower- and higher-level models.
To this end, we propose the following generalization of restriction set for soft interventions on an SCM\@.
\begin{definition}[Soft Restriction Set]
    The soft restriction $\softrst(\Mi)$ of a soft intervention $i = {(\bm{V}\gets\bm{f})}$ on a model $\M$ is the subset of total settings on ${\bm{X}\supseteq\bm{V}}$ matching any partial setting from the image of $\bm{f}$. 
    Formally,
    \begin{equation}
    \softrst(\M_i) = \{{\bm{x}\in\val(\bm{X})} \mid \proj(\bm{x},V) \in \img(\bm{f})\}
    \end{equation}
\end{definition}

Previous work introduced partial ordering between hard interventions according to constants assigned to intervened variables~\citep{rubenstein2017causal}.
Furthermore, preserving intervention ordering constitutes a desirable property of causal abstraction.
Therefore, we propose the following partial ordering on soft interventions.
\begin{definition}[Soft Interventions Ordering]
    Given two soft interventions $i_1$ and $i_2$ on an SCM~$\M$, the intervention $i_1$ precedes $i_2$ whenever the soft restriction of the former contains that of the latter. Formally,
    \begin{equation}
        i_1 \preceq i_2 \iff \softrst(\M_{i_1}) \supseteq \softrst(\M_{i_2}).
    \end{equation}
\end{definition}

Our definition of restriction for soft interventions is a strict generalization of the corresponding definition on hard interventions. Our partial ordering on soft interventions is also a strict generalization of the ordering on hard interventions. See Appendix~\ref{app:generalization} for proofs.

\subsection{Low-Level Soft Interventions}\label{subsec:lowsoftabstraction}

The definition of $\tau$-abstraction induces a unique transformation~$\omega$ between interventions whenever the restrictions of any pair of lower- and higher-level interventions are related by $\tau$.
Given our generalization, we can use soft restriction in place of strict restriction for a definition of abstraction immediately supporting low-level soft interventions.
\begin{definition}[Low Soft Abstraction]\label{def:lowsoftabstraction}
  Let $\L=\Ltuple$ and $\H=\Htuple$ be two SCMs, $\I$ and $\J$ be their respective sets of admissible interventions, $\tau_{\bm{Y}} \colon \val(\bm{X}) \to \val(\bm{Y})$ be a surjective function between endogenous settings, and $\tau_{\bm{U}} \colon \val(\bm{E}) \to \val(\bm{U})$ be a surjective function between exogenous settings.
  The model $\H$ is a low soft $\tau$-abstraction of $\L$ whenever there exists a surjective function ${\omega\colon\I\to\J}$ such that, for any intervention~$i \in\I$ it holds
  \begin{equation}\label{eq:lowsoftomega}
      \softrst(\H_{\omega(i)}) = \{\tau_{\bm{Y}}(\bm{x}) \mid \bm{x} \in \softrst(\L_i) \}
  \end{equation}
  and for every lower-level exogenous setting~$\bm{e}\in\val(\bm{E})$, it holds
  \begin{equation}\label{eq:lowsoftcompatibility}
      \tau_{\bm{Y}}(\L_i(\bm{e})) = \H_{\omega(i)}(\tau_{\bm{U}}(\bm{e})).
  \end{equation}
\end{definition}

\begin{figure}
\centering
\begin{tikzpicture}
{
  \node[draw,rectangle,rounded corners,thick,anchor=north west,minimum width=3.2cm,minimum height=2.1cm] at (0,0cm) (HBOX) {};
  \node[draw=none, above=of HBOX.west, anchor=south west, yshift=0.1cm](H) {\Large $\H$};
  \node[draw=none, below=of HBOX.north west, anchor=south west, yshift=0.2cm, xshift=0.1cm](Y1) {$Y_1 \coloneqq U_1$};
  \node[draw=none, below=of Y1.west, anchor=west, yshift=0.4cm](Y2) {$Y_2 \coloneqq U_2$};
  \node[draw=none, below=of Y2.west, anchor=west, yshift=0.4cm](Y3) {$Y_3 \coloneqq  Y_1 \lor Y_2$};
  
  \node[draw,rectangle,rounded corners,thick,anchor=north west,minimum width=3.4cm,minimum height=2.1cm,above=of HBOX,yshift=0.5cm] (LBOX) {};
  \node[draw=none, above=of LBOX.west, anchor=south west, yshift=0.1cm](L) {\Large $\L$};
  \node[draw=none, below=of LBOX.north west, anchor=south west, yshift=0.2cm, xshift=0.1cm](X1) {$X_1 \coloneqq E_1$};
  \node[draw=none, below=of X1.west, anchor=west, yshift=0.4cm](X2) {$X_2 \coloneqq E_2$};
  \node[draw=none, below=of X2.west, anchor=west, yshift=0.4cm](X3) {$X_3 \coloneqq X_1 X_2$};
  
  \node[draw,rectangle,rounded corners,thick,anchor=north west,minimum width=3.4cm,minimum height=2.1cm,right=of LBOX,xshift=2cm] (LiBOX) {};
  \node[draw=none, above=of LiBOX.west, anchor=south west, yshift=0.1cm](L) {\Large $\L_i$};
  \node[draw=none, below=of LiBOX.north west, anchor=south west, yshift=0.2cm, xshift=0.1cm](X1) {$X_1 \coloneqq E_1$};
  \node[draw=none, below=of X1.west, anchor=west, yshift=0.4cm](X2) {$X_2 \coloneqq 2 E_2$};
  \node[draw=none, below=of X2.west, anchor=west, yshift=0.4cm](X3) {$X_3 \coloneqq X_1X_2$};
  
  \node[draw,rectangle,rounded corners,thick,anchor=north west,minimum width=3.4cm,minimum height=2.1cm,right=of HBOX,xshift=2cm] (HiBOX) {};
  \node[draw=none, above=of HiBOX.west, anchor=south west, yshift=0.1cm](L) {\Large $\H_{\omega(i)}$};
  \node[draw=none, below=of HiBOX.north west, anchor=south west, yshift=0.2cm, xshift=0.1cm](Y1) {$Y_1 \coloneqq U_1$};
  \node[draw=none, below=of Y1.west, anchor=west, yshift=0.4cm](Y2) {$Y_2 \coloneqq \mathtt{T}$};
  \node[draw=none, below=of Y2.west, anchor=west, yshift=0.4cm](Y3) {$Y_3 \coloneqq  Y_1 \lor Y_2$};
  
  \draw[-{Latex}] (LBOX)--(LiBOX);
  \node[draw=none, yshift=0.3cm](INT) at ($(LBOX)!0.5!(LiBOX)$){$X_2 \gets 2 E_2$};
  
  \draw[-{Latex}] (HBOX)--(HiBOX);
  \node[draw=none, yshift=-0.3cm](OMEGAINT) at ($(HBOX)!0.5!(HiBOX)$){$Y_2 \gets \mathtt{T}$};

  \draw[-{Latex}] (LBOX)--(HBOX);
  \node[draw=none, xshift=0.3cm](TAUOBS) at ($(LBOX)!0.5!(HBOX)$){$\tau$};

  \draw[-{Latex}] (LiBOX)--(HiBOX);
  \node[draw=none,xshift=0.3cm] (TAUINT) at ($(LiBOX)!0.5!(HiBOX)$){$\tau$};

  \node[draw=none,xshift=0.3cm](EXPLOMEGA) at ($(INT)!0.5!(OMEGAINT)$){$\omega$};
  \draw[-{Latex}] (INT)--(OMEGAINT);
}
\end{tikzpicture}
\caption{
The low-level model~$\L$ multiplies two integers. The high level model~$\H$ computes whether one of two Boolean variables are true. The model $\H$ abstracts the model $\L$ with a map~$\tau$ that computes the parity of each integer at the low-level, where $X_n$ determines $Y_n$. This is because when two integers are multiplied, the product is even if and only if at least one of the integers is even. The soft intervention~$i ={(X_2 \gets 2E_2)}$ doubles the value of a low-level variable $X_2$ making it always even. Therefore, it is mapped to a constant hard intervention $\omega(i) = {(Y_2\gets \mathtt{T})}$ that fixes the high-level variable $Y_2$ to be always true.
}\label{fig:lowsoftabstraction}
\end{figure}
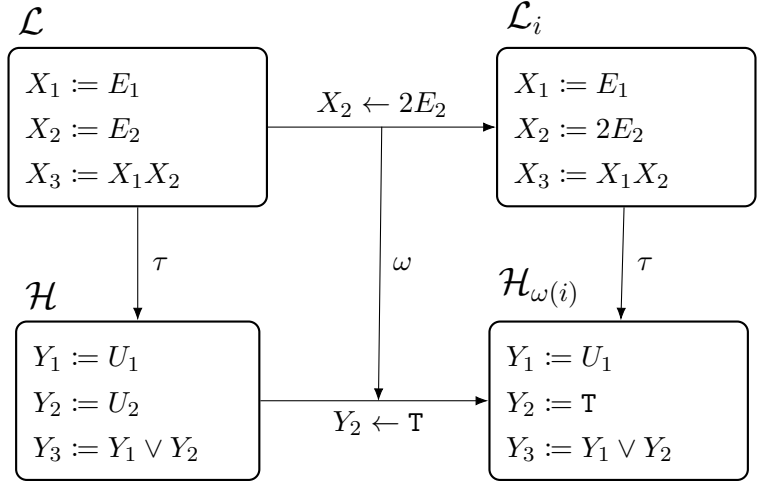

As proved in Appendix~\ref{app:lowersoftabstraction}, by assuming that the high-level model admits only hard interventions, the induced mapping~$\omega$ is uniquely defined, order-preserving, and has a fixed point in the empty intervention.
Furthermore, whenever admissible interventions on both models are hard, low soft $\tau$-abstraction reduces to $\tau$-abstraction. 
As the following example shows, Definition~\ref{def:lowsoftabstraction} effectively represents abstraction whenever a low-level soft intervention results in an hard intervention on the abstract model. 
\begin{example}\label{ex:lowsoftabstraction}
Let $\L=\Ltuple$ and $\H=\Htuple$ be two SCMs with structural equations as reported in Figure~\ref{fig:lowsoftabstraction}.
At the lower-level, every endogenous or exogenous variable $X\in{\bm{X}\cup\bm{E}}$ is integer with domain $\val(X)=\mathbb{Z}$.
The high-level model $\H$ consists instead of Boolean variables $Y\in{\bm{Y}\cup\bm{U}}$ with domain $\val(Y)=\mathbb{B}$.
Given a function $\tau$ composed by element-wise mapping a function $\even\colon\mathbb{Z}\to\mathbb{B}$ testing the parity of each integer, $\H$ is a $\tau$-abstraction of $\L$ for every hard intervention at the lower-level.
Furthermore, any soft intervention at the lower-level that always produces either even or odd integers also results in an abstract hard intervention.
For instance, let us take the soft intervention $i=({X_2\gets 2E_2})$: its soft restriction $\softrst(\L_i)$ is the set of integer triples $\{(x_1, x_2, x_3) \mid \even(x_2) = \mathtt{T}\}$.
Consequently, the result of applying $\tau$ to the soft restriction of $i$ is the set of Boolean triples $\{y_1, y_2, y_3 \mid y_2 = \mathtt{T}\}$.
The set coincides exclusively with the restriction of the high-level hard intervention $j = (Y_2\gets \mathtt{T})$.
Furthermore, the consistency requirement holds as in
\begin{equation}
    \begin{split}
        \tau_{\bm{Y}}(\L_i([e_1,e_2])) &= \tau_Y([e_1, 2e_2, e_1e_2])\\
        &= [\even(e_1), \mathtt{T}, \even(e_1e_2)]\\
        &= [\even(e_1), \mathtt{T}, \even(e_1) \lor \even(e_2)]\\
        &= \H_j(\tau_{\bm{U}}([e_1, e_2])).
    \end{split}
\end{equation}
Therefore, $\H$ low soft abstracts $\L$ and $\omega$ uniquely maps the low level intervention $i$ to $j$.
\end{example}

Crucially, the intervention map $\omega$ is only guaranteed to be unique whenever the set of admissible interventions at the higher level contains exclusively hard interventions. The definition of abstraction provided by \citet{beckers2019abstracting} ensures consistency for all exogenous settings. Nonetheless, whenever structural equations are non-surjective, consistency is not enforced on all endogenous settings. Therefore, an intervention on the low level model might be mapped to multiple distinct abstract soft interventions that coincide for any exogenous-induced endogenous setting. Consequently, as proved in the following theorem, there exist different maps $\omega$ for each possible choice of high-level intervention.
\begin{theorem}[Ambiguity in Low Soft Abstraction]\label{theo:ambiguity}
    Let $\H=\Htuple$ be an SCM with admissible interventions $\J$.
    Given a variable ${V\in\bm{Y}}$, if $\J$ contains two distinct interventions $j ={(V \gets g)}$ and $j' ={(V \gets g')}$ such that
    \begin{enumerate}
    \item $\forall{\bm{u}\in\val(\bm{U})} \colon \proj(\H_j(\bm{u}), V) =
    \proj(\H_{j'}(\bm{u}), V)$,
    \item $\img(g) = \img(g')$,
    \end{enumerate}
    then, for any causal model~$\L$ with admissible interventions $\I$, whenever $\H$ low soft $\tau$-abstracts $\L$, the corresponding function $\omega\colon\I\to\J$ is not uniquely defined. 
\end{theorem}
\begin{proof}
    We report the proof in Appendix~\ref{app:ambiguity}.
\end{proof}

In the following, we alter the previous example to report a scenario where non-surjective structural equations lead to a possibly non-unique intervention map~$\omega$.
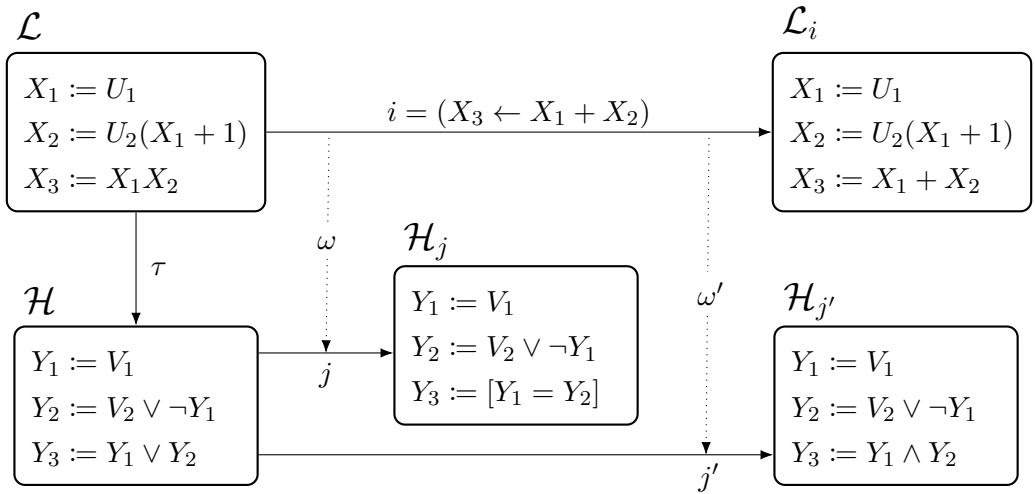
\begin{figure}[t]
\centering
\begin{tikzpicture}
{
  \node[draw,rectangle,rounded corners,thick,anchor=north west,minimum width=3.2cm,minimum height=2.1cm] at (0,0cm) (HBOX) {};
  \node[draw=none, above=of HBOX.west, anchor=south west, yshift=0.1cm](H) {\Large $\H$};
  \node[draw=none, below=of HBOX.north west, anchor=south west, yshift=0.2cm, xshift=0.1cm](Y1) {$Y_1 \coloneqq V_1$};
  \node[draw=none, below=of Y1.west, anchor=west, yshift=0.4cm](Y2) {$Y_2 \coloneqq V_2 \lor \neg Y_1$};
  \node[draw=none, below=of Y2.west, anchor=west, yshift=0.4cm](Y3) {$Y_3 \coloneqq  Y_1 \lor Y_2$};
  
  \node[draw,rectangle,rounded corners,thick,anchor=north west,minimum width=3.2cm,minimum height=2.1cm,right of=HBOX,xshift=4cm,yshift=0.8cm] (0,0cm) (HjBOX) {};
  \node[draw=none, above=of HjBOX.west, anchor=south west](Hj) {\Large $\H_j$};
  \node[draw=none, below=of HjBOX.north west, anchor=south west, yshift=0.2cm, xshift=0.1cm](Y1) {$Y_1 \coloneqq V_1$};
  \node[draw=none, below=of Y1.west, anchor=west, yshift=0.4cm](Y2) {$Y_2 \coloneqq V_2 \lor \neg Y_1$};
  \node[draw=none, below=of Y2.west, anchor=west, yshift=0.4cm](Y3) {$Y_3 \coloneqq  [Y_1 = Y_2]$};
  
  \node[draw,rectangle,rounded corners,thick,anchor=north west,minimum width=3.2cm,minimum height=2.1cm,right of=HjBOX,xshift=4cm,yshift=-0.8cm] (0,0cm) (HjprimeBOX) {};
  \node[draw=none, above=of HjprimeBOX.west, anchor=south west](Hjprime) {\Large $\H_{j'}$};
  \node[draw=none, below=of HjprimeBOX.north west, anchor=south west, yshift=0.2cm, xshift=0.1cm](Y1) {$Y_1 \coloneqq V_1$};
  \node[draw=none, below=of Y1.west, anchor=west, yshift=0.4cm](Y2) {$Y_2 \coloneqq V_2 \lor \neg Y_1$};
  \node[draw=none, below=of Y2.west, anchor=west, yshift=0.4cm](Y3) {$Y_3 \coloneqq  Y_1 \land Y_2$};
  
  \node[draw,rectangle,rounded corners,thick,anchor=north west,minimum width=3.4cm,minimum height=2.1cm,above=of HBOX,yshift=0.5cm] (LBOX) {};
  \node[draw=none, above=of LBOX.west, anchor=south west, yshift=0.1cm](L) {\Large $\L$};
  \node[draw=none, below=of LBOX.north west, anchor=south west, yshift=0.2cm, xshift=0.1cm](X1) {$X_1 \coloneqq U_1$};
  \node[draw=none, below=of X1.west, anchor=west, yshift=0.4cm](X2) {$X_2 \coloneqq U_2(X_1+1)$};
  \node[draw=none, below=of X2.west, anchor=west, yshift=0.4cm](X3) {$X_3 \coloneqq  X_1 X_2$};
  
  \node[draw,rectangle,rounded corners,thick,anchor=north west,minimum width=3.4cm,minimum height=2.1cm,right=of LBOX,xshift=5.65cm] (LiBOX) {};
  \node[draw=none, above=of LiBOX.west, anchor=south west, yshift=0.1cm](L) {\Large $\L_i$};
  \node[draw=none, below=of LiBOX.north west, anchor=south west, yshift=0.2cm, xshift=0.1cm](X1) {$X_1 \coloneqq U_1$};
  \node[draw=none, below=of X1.west, anchor=west, yshift=0.4cm](X2) {$X_2 \coloneqq U_2(X_1+1)$};
  \node[draw=none, below=of X2.west, anchor=west, yshift=0.4cm](X3) {$X_3 \coloneqq  X_1 + X_2$};
  
  \draw[-{Latex}] (LBOX)--(HBOX);
  \node[draw=none, xshift=0.3cm](TAUOBS) at ($(LBOX)!0.5!(HBOX)$){$\tau$};
  
  \draw[-{Latex}] (LBOX)--(LiBOX);
  \node[draw=none, yshift=0.25cm](INT) at ($(LBOX)!0.5!(LiBOX)$){$i={(X_3\gets X_1+X_2)}$};
  
  \node[draw=none,yshift=0.7cm] at (HBOX.east) (A) {};
  \node[draw=none,yshift=-0.1cm] at (HjBOX.west) (B) {};
  \node[draw=none,yshift=-0.65cm] at (HBOX.east) (C) {};
  \node[draw=none,yshift=-0.65cm] at (HjprimeBOX.west) (D) {};
  \node[draw=none] at (LBOX.east) (E) {};
  
  \node[draw=none, yshift=-0.3cm](JINT) at ($(A)!0.5!(B)$){$j$};
  \draw[-{Latex}] ([yshift=0.7cm]HBOX.east)--([yshift=-0.1cm]HjBOX.west);
  \node[draw=none, yshift=-0.3cm](JPRIMEINT) at ($(C)!0.875!(D)$){$j'$};
  \draw[-{Latex}] ([yshift=-0.65cm]HBOX.east)--([yshift=-0.65cm]HjprimeBOX.west);
  \draw[-{Latex},dotted] ([xshift=-2.5cm]INT.south)--(JINT.north);
  \node[draw=none,fill=white, xshift=0.85cm](OMEGA) at ($(E)!0.5!(A)$){$\omega$};
  \draw[-{Latex},dotted] ([xshift=2.5cm]INT.south)--([xshift=5cm,yshift=-1.35cm]JINT.north);
  \node[draw=none,fill=white, xshift=5.9cm](OMEGAPRIME) at ($(E)!0.5!(C)$){$\omega'$};
}
\end{tikzpicture}
\caption{
The low-level model~$\L$ performs arithmetic operations on two integers.
The high level model~$\H$ computes a Boolean proposition.
The model $\H$ abstracts the model $\L$ with a map~$\tau$ that computes the parity of each integer at the low-level, where $X_n$ determines $Y_n$. 
Given the intervention $i={(X_3\gets X_1+X_2)}$, there exist two distinct intervention maps satisfying low soft abstraction: $\omega(i) = {(Y_3 \gets {[Y_1 =  Y_2]})}$ and $\omega'(i)={(Y_3 \gets {Y_1 \land Y_2})}$, where the former tests equality.
Intuitively, this occurs since the two interventions differ only when $Y_1$ and $Y_2$ are both false, but this condition is never reached for any exogenous setting.
}\label{fig:ambiguousomega}
\end{figure}

\begin{example}\label{ex:ambiguosomega}
Let $\L=\Ltuple$ and $\H=\Htuple$ be two SCMs with structural equations as reported in Figure~\ref{fig:ambiguousomega}.
As in previous Example~\ref{ex:lowsoftabstraction}, Boolean-valued model~$\H$ $\tau$-abstracts integer-valued model~$\L$ given a function $\tau$ element-wise testing the parity of each integer.
Differently from before, the structural equations of $X_1$ and $X_2$ are jointly not surjective.
In fact, for any exogenous setting, $X_1$ and $X_2$ will never be both odd.
Similarly, at the higher-level, variables $Y_1$ and $Y_2$ will never be both false.
We define a low-level soft intervention
$i={(X_3\gets X_1+X_2)}$ and two high-level soft interventions
$j={(Y_3\gets [Y_1 = Y_2])}$, $j'={(Y_3\gets Y_1 \land Y_2)}$.
Exploiting the unrealized endogenous values, we can show that exist a function~$\omega$ mapping $i$ to $j$ and a function~$\omega'$ mapping $i$ to $j'$, both satisfying the requirements for low soft abstraction.

The replacement function in intervention~$i$ outputs even and odd numbers and the functions for $j$ and $j'$ output both true and false values, so we know that
\[\softrst(\H_{j'}) = \softrst(\H_{j}) = \{\tau_{\bm{Y}}(\bm{x}) \mid \bm{x} \in \softrst(\L_i)\}.\]
Furthermore, for any low-level exogenous setting $\bm{e}\in\val(\bm{E})$, it holds that
\begin{align*}
    \tau_{\bm{Y}}(\L_i(\bm{e})) &= \M_j(\tau_{\bm{U}}(\bm{e}))\\
    \tau_{\bm{Y}}(\L_i(\bm{e})) &= \M_{j'}(\tau_{\bm{U}}(\bm{e})).
\end{align*}
Consequently, given a function~$\omega$ witnessing that $\H$ low soft $\tau$-abstracts $\L$, we can construct a distinct map~${\omega'}$ that differs on input $i$ and also witnesses the same abstraction relationship.
\end{example}

\subsection{High-Level and Low-Level Soft Interventions}\label{subsec:softabstraction}

Intuitively, ambiguity in low soft abstraction arises whenever there exist two distinct soft interventions on a variable whose parents have non-surjective structural equations. Therefore, two distinct functions could be assigned to the same variable if they agree at least on the image of the structural equations of their parents. We address this problem by requiring consistency not only for all exogenous settings, but also for every endogenous settings.
\begin{definition}[Soft Abstraction]\label{def:softabstraction}
  Let $\L=\Ltuple$ and $\H=\Htuple$ be two SCMs, $\I$ and $\J$ be their respective sets of admissible interventions, $\tau_{\bm{Y}} \colon \val(\bm{X}) \to \val(\bm{Y})$ be a surjective function between endogenous settings, and $\tau_{\bm{U}} \colon \val(\bm{E}) \to \val(\bm{U})$ be a surjective function between exogenous settings.
  The model $\H$ is a soft $\tau$-abstraction of $\L$ whenever it exists a surjective function ${\omega\colon\I\to\J}$ such that, for any intervention~$i \in\I$ it holds
  \begin{equation}\label{eq:softomega}
      \softrst(\H_{\omega(i)}) = \{\tau_{\bm{Y}}(\bm{x}) \mid \bm{x} \in \softrst(\L_i) \}
  \end{equation}
  and for every lower-level exogenous~$\bm{e}\in\val(\bm{E})$ and endogenous~$\bm{x}\in\val(\bm{X})$ setting, it holds
  \begin{equation}\label{eq:softcompatibility}
    \tau_{\bm{Y}}(\bm{F}^i(\bm{x}, \bm{e})) = 
    \bm{G}^{\omega(i)}(\tau_{\bm{Y}}(\bm{x}), \tau_{\bm{U}}(\bm{e})).
   \end{equation}
\end{definition}

Given this definition, the intervention mapping $\omega$ is unique for every combination of hard and soft interventions. Notably, our definition specializes to low soft $\tau$-abstraction when all high-level interventions are hard. Similarly, soft $\tau$-abstraction implies $\tau$-abstraction, and vice versa, when all admitted interventions are hard interventions. We report in Appendix~\ref{app:softabstraction} the proofs on uniqueness, fixed point, and order-preserving properties of the intervention mapping $\omega$ in soft abstraction. Further, we continue previous Example~\ref{ex:ambiguosomega} to highlight how soft abstraction disambiguates high-level soft interventions.
\begin{example}\label{ex:disambiguosomega}
Recall the previous Example~\ref{ex:ambiguosomega}, where there were two intervention maps $\omega$ and $\omega'$ that, given the two SCMs in Figure~\ref{fig:ambiguousomega}, witnessed $\H$ being an abstraction of $\L$.
Given the exogenous setting $\bm{e} = (1,1)$ \textit{and} the endogenous setting $\bm{x} = (1, 1, 1)$, it is immediate that
\begin{equation}
\begin{split}
    \tau_{\bm{Y}}(\bm{F}^i(\bm{x}, \bm{e})) &= [\mathtt{F}, \mathtt{T}, \mathtt{T}] \\ 
    \bm{G}^{\bm{j}}(\tau_{\bm{Y}}(\bm{x}, \tau_{\bm{U}}(\bm{e})))
    &= [\mathtt{F}, \mathtt{T}, \mathtt{T}]\\
    \bm{G}^{\bm{j'}}(\tau_{\bm{Y}}(\bm{x}, \tau_{\bm{U}}(\bm{e})))
    &= [\mathtt{F}, \mathtt{T}, \mathtt{F}].
\end{split}
\end{equation}
Therefore, mapping the intervention $i$ to $j'$ would break consistency on the provided setting.
Thus, we no longer have an ambiguous intervention map~$\omega$.
\end{example}

\section{Constructive Soft Abstraction}\label{sec:constructivesoftabstraction}

Our definition of Soft Abstraction relates two distinct SCMs at different levels of detail according to their response to possibly soft interventions. In this way, we can determine abstraction whenever an intervention at the lower level produces a non-constant effect at the higher level, vice versa, or both. One critical result is the uniqueness of the function mapping interventions between SCMs. Therefore, whenever a model abstracts another, each lower-level manipulation corresponds to a unique abstract intervention. In this section, we build on the uniqueness of the intervention mapping to derive its closed and explicit form.

To this end, we require grouping endogenous and exogenous variables on the lower-level model. Therefore, we consider the existence of an alignment function mapping high-level variables into subsets of low-level variables.
\begin{definition}[Alignment]
    Given two SCMs $\H=\Htuple$ and $\L=\Ltuple$
    such that $\H$ soft $\tau$-abstracts $\L$,
    we define an alignment to be a map $\Pi \colon {\bm{Y}\cup\bm{U}}\to \mathcal{P}(\bm{X}\cup\bm{E})$
    that satisfies the following:
    (1)
    There exist a set of functions
    ${\{\tau_Y\}}_{Y \in \bm{Y}}$
    such that,
    for each endogenous high-level variable~${Y\in\bm{Y}}$
    and low-level endogenous setting~${\bm{x}\in\val(\bm{X})}$,
    it holds
    $\proj(\tau_{\bm{Y}}(\bm{x}), Y ) = \tau_Y(\proj(\bm{x}, \Pi(Y)))$,
    where $\Pi(Y)\subseteq\bm{X}$
    and ${\tau_Y\colon\val(\Pi(Y))\to\val(Y)}$;
    (2)
    There exist a set of functions
    ${\{\tau_U\}}_{U \in \bm{U}}$
    such that,
    for each exogenous high-level variable ${U\in\bm{U}}$
    and low-level exogenous setting~${\bm{e}\in\val(\bm{E})}$
    it holds
    $\proj(\tau_{\bm{U}}(\bm{e}), U ) = \tau_U(\proj(\bm{e}, \Pi(U)))$,
    where $\Pi(U)\subseteq\bm{E}$
    and ${\tau_U\colon\val(\Pi(U))\to\val(U)}$.
\end{definition}

Essentially, we require the abstraction function $\tau_Y$ of an high-level endogenous variable $Y$ to depend only on a subset of low-level endogenous variables $\Pi(Y)$. Similarly, we require each high-level exogenous variable to depend only on a subset of low-level exogenous variables. Assuming the existence of a map from high-level to low-level variables is coherent with previous definitions of constructive abstraction~\citep{beckers2019abstracting} and model transformation~\citep{rischel2021compositional} for hard-intervened causal models. In particular, constructive abstraction further specializes causal abstraction to address the problem of clustering micro-variables into distinct causal variables at the higher-level. Given our definition of Soft Abstraction, a similar extension is immediate in our framework. Therefore, we define Constructive Soft Abstraction whenever the abstraction partitions lower-level variables.
\begin{definition}[Constructive Soft Abstraction]
    An SCM~$\H=\Htuple$ is a constructive soft $\tau$-abstraction of another SCM~$\L=\Ltuple$ whenever $\H$ $\tau$-abstracts $\L$ and there exists an alignment $\Pi$ partitioning a subset of low-level variables.
\end{definition}

Since constructive soft abstraction extends soft abstraction, the function~$\omega$ between interventions is unique. Therefore, we seek an explicit form of $\omega$ returning, for an intervention $(\Pi(Y)\gets\bm{f})$ on the partition of an high-level variable $Y$, the corresponding higher-level intervention $({Y\gets{g}})$.
With the following theorem, we prove that the target function $g_Y$ necessarily consists of the composition of three operations: (i) mapping the inputs of the function on the partitions of the parents $\Pi(\pa(Y))$; (ii) computing the value of the partition $\Pi(Y)$ on the intervened model $\L_i$; and (iii) finally mapping back the partition to the higher-level variable $Y$ (Figure~\ref{fig:explicitomega}).
\begin{theorem}[Explicit Intervention Map]\label{theo:explicitomega}
  Let $\L=\Ltuple$ and $\H=\Htuple$ be two SCMs and $\I$ and $\J$ their corresponding sets of admissible interventions.
  Whenever $\H$ soft {$\tau$-abstracts} $\L$ with alignment $\Pi$,
  $\omega$ maps the low-level intervention $i ={(\bm{V} \leftarrow \bm{f})}$
  to an intervention $j ={(\bm{W} \leftarrow \bm{g})}$
  on the high-level variables 
  $\bm{W} \subseteq \{ Y \mid \bm{V} \cap \Pi(Y) \neq \emptyset \}$
  such that
  \begin{equation}\label{eq:explicitomega}
      g_Y(\bm{y}, \bm{u}) = 
      \tau_Y(F^i_{\Pi(Y)}(
      \tau^{-1}_{\bm{Y}}(\bm{y}),
      \tau^{-1}_{\bm{U}}(\bm{u}))),
  \end{equation}
  for any variable $Y\in\bm{W}$ and any partial inverse $\tau^{-1}$.
\end{theorem}
\begin{proof}
We report the proof in Appendix~\ref{app:explinttrans}.
\end{proof}

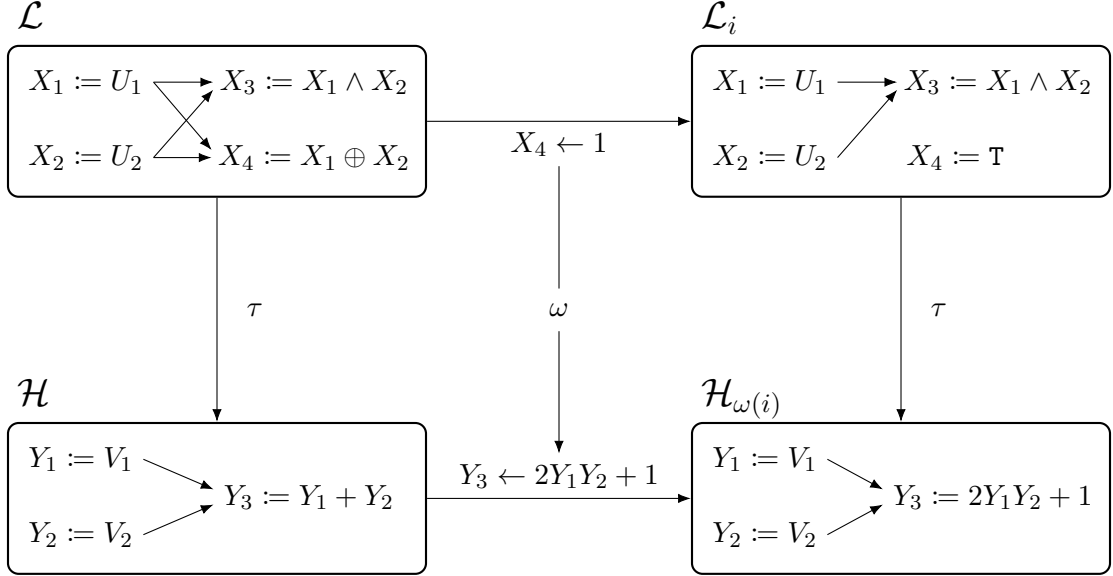
\begin{figure}
    \centering
\begin{tikzpicture}
{
  \normalsize

  \node[draw,rectangle,rounded corners,thick,anchor=north west,minimum width=5.5cm,minimum height=2cm] at (0cm,0cm) (LBOX) {};
  \node[draw=none, anchor=north west](L) at (0cm, 0.7cm) {\Large $\L$};
  \node[draw=none, anchor=north west](X1) at (0.15cm, -0.2cm) {$X_1 \coloneqq U_1$};
  \node[draw=none, below of=X1](X2) {$X_2 \coloneqq U_2$};
  \node[draw=none, right of=X1, xshift=2cm](X3) {$X_3 \coloneqq X_1 \land X_2$};
  \node[draw=none, below of=X3](X4) {$X_4 \coloneqq X_1 \oplus X_2$};
  \draw[-{Latex}] (X1.east)--([xshift=0.05cm]X3.west);
  \draw[-{Latex}] (X1.east)--([xshift=0.05cm,yshift=0.1cm]X4.west);
  \draw[-{Latex}] (X2.east)--([xshift=0.05cm,yshift=-0.1cm]X3.west);
  \draw[-{Latex}] (X2.east)--([xshift=0.05cm]X4.west);

  \node[draw,rectangle,rounded corners,thick,anchor=north west,minimum width=5.5cm,minimum height=2cm] at (0cm,-5cm) (HBOX) {};
  \node[draw=none, anchor=north west](H) at (0cm, -4.30cm) {\Large $\H$};
  \node[draw=none, anchor=north west](Y1) at (0.15cm, -5.2cm) {$Y_1 \coloneqq V_1$};
  \node[draw=none, below of=Y1](Y2) {$Y_2 \coloneqq V_2$};
  \node[draw=none, right of=Y1, xshift=2cm, yshift=-0.5cm](Y3) {$Y_3 \coloneqq Y_1 + Y_2$};
  \draw[-{Latex}] (Y1.east)--([yshift=0.1cm]Y3.west);
  \draw[-{Latex}] (Y2.east)--([yshift=-0.1cm]Y3.west);

  \node[draw,rectangle,rounded corners,thick,anchor=north west,minimum width=5.5cm,minimum height=2cm] at (9cm,0cm) (LiBOX) {};
  \node[draw=none, anchor=north west](Li) at (9cm, 0.7cm) {\Large $\L_i$};
  \node[draw=none, anchor=north west](Xi1) at (9.15cm, -0.2cm) {$X_1 \coloneqq U_1$};
  \node[draw=none, below of=Xi1](Xi2) {$X_2 \coloneqq U_2$};
  \node[draw=none, right of=Xi1, xshift=2cm](Xi3) {$X_3 \coloneqq X_1 \land X_2$};
  \node[draw=none, right of=Xi2, xshift=1.44cm](Xi4) {$X_4 \coloneqq \mathtt{T}$};
  \draw[-{Latex}] (Xi1.east)--([xshift=0.05cm]Xi3.west);
  \draw[-{Latex}] (Xi2.east)--([xshift=0.05cm,yshift=-0.1cm]Xi3.west);

  \node[draw,rectangle,rounded corners,thick,anchor=north west,minimum width=5.5cm,minimum height=2cm] at (9cm,-5cm) (HiBOX) {};
  \node[draw=none, anchor=north west](Hi) at (9cm, -4.30cm) {\Large $\H_{\omega(i)}$};
  \node[draw=none, anchor=north west](Y1) at (9.15cm, -5.2cm) {$Y_1 \coloneqq V_1$};
  \node[draw=none, below of=Y1](Y2) {$Y_2 \coloneqq V_2$};
  \node[draw=none, right of=Y1, xshift=2cm, yshift=-0.5cm](Y3) {$Y_3 \coloneqq 2Y_1Y_2 + 1$};
  \draw[-{Latex}] (Y1.east)--([yshift=0.1cm]Y3.west);
  \draw[-{Latex}] (Y2.east)--([yshift=-0.1cm]Y3.west);

  \draw[-{Latex}] (LBOX)--(LiBOX);
  \node[draw=none, yshift=-0.3cm](INT) at ($(LBOX)!0.5!(LiBOX)$){$X_4 \gets 1$};

  \draw[-{Latex}] (HBOX)--(HiBOX);
  \node[draw=none, yshift=0.3cm](OMEGAINT) at ($(HBOX)!0.5!(HiBOX)$){$Y_3 \gets 2Y_1Y_2+1$};

  \draw[-{Latex}] (LBOX)--(HBOX);
  \node[draw=none, xshift=0.5cm](TAUOBS) at ($(LBOX)!0.5!(HBOX)$){$\tau$};

  \draw[-{Latex}] (LiBOX)--(HiBOX);
  \node[draw=none, xshift=0.5cm](TAUINT) at ($(LiBOX)!0.5!(HiBOX)$){$\tau$};

  \node[draw=none,outer sep=0.05cm, inner sep=0.15cm](EXPLOMEGA) at ($(INT)!0.5!(OMEGAINT)$) {$\omega$};
  \draw[-] (INT)--(EXPLOMEGA);
  \draw[-{Latex}] (EXPLOMEGA)--(OMEGAINT);
}
\end{tikzpicture}
\caption{
The SCM~$\L$ models a binary adder of two exogenous bits $U_1, U_2$, where $\oplus$ is the Boolean XOR operator. The model $\H$ adds two integers and stores the sum in a single integer $Y_3$. Performing an hard intervention on $\L$, for instance by fixing the least significant bit $X_4 \gets \mathtt{T}$, results in a soft intervention at the higher-level $Y_3 \gets 2Y_1Y_2+1$.
}\label{fig:intro}
\end{figure}
Remarkably, Theorem~\ref{theo:explicitomega} does not require the alignment $\Pi$ to partition low-level variables. Consequently, it applies to constructive soft abstraction as a special case. Furthermore, whenever both sets of admissible interventions are hard, the explicit form is a necessary condition for constructive $\tau$-abstraction. In the following example, we showcase how the explicit definition of the intervention map~$\omega$ can be applied to determine high-level interventions.
\begin{example}
Let $\L=\Ltuple$ and $\H=\Htuple$ be two SCMs with structural equations as reported in Figure~\ref{fig:intro}. At the low-level, every endogenous and exogenous variable $X\in{\bm{X}\cup\bm{E}}$ has Boolean domain $\val(X)= \{\mathtt{F},\mathtt{T}\}$. At the high level, $\val(Y_3) = \{0, 1, 2, 3\}$, while every other exogenous or endogenous variable $Y \in \bm{Y} \cup \bm{U} \setminus \{Y_3\}$ has domain $\val(Y) = \{0,1\}$. We consider the alignment $\Pi$ where $\Pi(Y_3) = \{X_3, X_4\}$, while $\Pi(Y_n) = \{X_n\}$ for every other high-level exogenous or endogenous variable $Y_n$. Let $\varphi$ be a function that maps true values to the integer $1$, and false values to $0$. We define the abstraction as $\tau_{Y_3}(\bm{x}) = 2\varphi(x_3) + \varphi(x_4)$, $\tau_{Y_1}(\bm{x}) = \varphi(x_1)$, and $\tau_{Y_2}(\bm{x}) = \varphi(x_2)$. Given the low-level intervention $i={(X_4 \gets \mathtt{T})}$, the corresponding high-level intervention is $\omega(i) = (Y_3\gets g)$, where
\begin{equation}
    \begin{split}
        g(\bm{y}, \bm{u})
        &= \tau_{Y_3}(F^i_{\Pi(Y_3)}(\tau^{-1}_{\bm{Y}}(\bm{y}),\tau^{-1}_{\bm{U}}(\bm{u})))\\
        &= \tau_{Y_3}(F^i_{\{X_3,X_4\}}((\varphi^{-1}(y_1), \varphi^{-1}(y_2), x_3, x_4),\bm{e}))\\
        &= \tau_{Y_3}(\varphi^{-1}(y_1) \land \varphi^{-1}(y_2), \mathtt{T})\\
        &= 2\varphi(\varphi^{-1}(y_1) \land \varphi^{-1}(y_2)) + \varphi(\mathtt{T})\\
        &= 2 y_1 y_2 + 1
    \end{split}
\end{equation}
Since $g$ depends only on its parents, the explicit form works for any choice of the remaining endogenous $\{X_3, X_4\}$ or exogenous $\bm{E}$ lower-level variables.
\end{example}

\section{Conclusion}

Hard interventions constitute one of the core concepts of causality. Existing theories of abstraction build on the assumption that a causal model abstracts another whenever there is a correspondence between their sets of admissible hard interventions. We generalize this theory to soft interventions to support non-constant change at different levels of abstraction. To this end, we prove how generalizing constructs based on hard interventions in the definition of $\tau$-abstraction from \citet{beckers2019abstracting} does not suffice. Therefore, we propose to enforce causal consistency for all exogenous \textit{and} endogenous settings. Together with our generalization effort, this strengthened requirement results in our novel definition of soft abstraction. As proved and exemplified, soft abstraction induces a unique function~$\omega$ mapping low-level interventions to high-level interventions in the soft-intervened scenario. Finally, we report that our constructive definition of soft abstraction implies a unique explicit form for the intervention map~$\omega$. Since soft abstraction correctly reduces to $\tau$-abstraction in the hard-intervened scenario, this result effectively extends previous formulations that exclusively proved the uniqueness of~$\omega$. By demonstrating the explicit form as a necessary condition, our contribution opens up practical applications for assessing abstraction without requiring an evaluation of each exogenous or endogenous configuration.

\bibliography{softint}

\clearpage
\appendix
\section{Proofs}\label{app:proofs}

\subsection{Soft Generalization of Hard Constructs}\label{app:generalization}

Given an SCM~$\M$, for any hard intervention $i$ on the model, it holds that
\begin{enumerate}
    \item $\softrst(M_i) = \rst(M_i)$,
\begin{proof}
Given a subset ${\bm{V}\subseteq\bm{X}}$ of endogenous variables, a hard intervention $i=(\bm{V} \gets \bm{v})$ constraints each intervened variable $X$ to a constant value $v$.
Equivalently, we can consider the same intervention as a function
$\bm{f} \colon \val(\bm{X}) \times \val(\bm{E}) \to \val(\bm{V})$ whose output is the constant $\bm{v}$. Therefore, $\img(\bm{f}) = \{v\}$ and, consequently,
\begin{equation}
\begin{split}
    \softrst(\M_i) &= \{{\bm{x}\in\val(\bm{X})} \mid \proj(\bm{x},V) \in \img(\bm{f})\}\\
    &= \{{\bm{x}\in\val(\bm{X})} \mid \proj(\bm{x},V) = \bm{v}\}\\
    &= \rst(\M_i).
\end{split}
\end{equation}
\end{proof}
    \item For any hard intervention $i'$, $i \sqsubseteq i' \iff i \preceq i'$.
\begin{proof}
Given the previous property, we prove that soft intervention ordering implies hard intervention ordering for any pair of hard interventions $i = (\bm{V} \gets \bm{v})$ and $i' = (\bm{W} \gets \bm{w})$. In fact,
\begin{equation}
    \begin{split}
        i \preceq i'
        &\iff \softrst(\M_i) \supseteq \softrst(\M_{i'}) \\
        &\iff \rst(\M_i) \supseteq \rst(\M_{i'}),\\
    \end{split}
\end{equation}
whenever $i,i'$ are hard interventions.
If the restriction of $\M_i$ contains that of $\M_{i'}$, then $\M_i$ allows more values for some variables. Equivalently, $i'$ intervenes on a larger set of variables than $i$ and sets them to the same constants, otherwise they would not have been in the restriction of $\M_i$. Consequently,
\begin{equation}
    \begin{split}
            \rst(\M_i) \supseteq \rst(\M_{i'})
            &\iff \bm{V}\subseteq\bm{W} \land \bm{v} = \proj(\bm{w}, \bm{V})\\
            &\iff i \sqsubseteq i'.
    \end{split}
\end{equation}
\end{proof}
\end{enumerate}

\subsection{Low Soft Abstraction}\label{app:lowersoftabstraction}

\textbf{Definition~\ref{def:lowsoftabstraction}}:
Let $\L=\Ltuple$ and $\H=\Htuple$ be two SCMs, $\I$ and $\J$ be their respective sets of admissible interventions, $\tau_{\bm{Y}} \colon \val(\bm{X}) \to \val(\bm{Y})$ be a surjective function between endogenous settings, and $\tau_{\bm{U}} \colon \val(\bm{E}) \to \val(\bm{U})$ be a surjective function between exogenous settings.
The model $\H$ is a low soft $\tau$-abstraction of $\L$ whenever it exists a surjective function ${\omega\colon\I\to\J}$ such that, for any intervention~$i \in\I$ it holds
\begin{equation}
  \softrst(\H_{\omega(i)}) = \{\tau_{\bm{Y}}(\bm{x}) \mid \bm{x} \in \softrst(\L_i) \}
\end{equation}
and for every lower-level exogenous setting~$\bm{e}\in\val(\bm{E})$, it holds
\begin{equation}
  \tau_{\bm{Y}}(\L_i(\bm{e})) = \H_{\omega(i)}(\tau_{\bm{U}}(\bm{e})).
\end{equation}

Given Definition~\ref{def:lowsoftabstraction} of Low Soft Abstraction, we prove that, whenever the set of admissible high-level intervention $\J$ contains only hard interventions and $\H$ low soft $\tau$-abstracts $\L$, it holds that
\begin{enumerate}
  \item $\omega$ is unique,
    \begin{proof}
    For any distinct surjective function $\omega'\neq\omega$, $\H$ does not $\tau$-abstracts $\L$. Since $\omega, \omega'$ are defined on the domain $\I$, they are distinct whenever there exist at least one intervention $i\in\I$ such that \[\omega(i)\neq\omega(i').\]
    Thus, there must exist two distinct high-level interventions $j, j' \in\H$ such that,
    \begin{align}
        \softrst(\H_j) &= \{\tau_{\bm{Y}}(\bm{x}) \mid \bm{x} \in \softrst(\L_i)\}\\
        \softrst(\H_{j'}) &= \{\tau_{\bm{Y}}(\bm{x}) \mid \bm{x} \in \softrst(\L_i)\}
    \end{align}
    Consequently, it must hold that
    \begin{equation}
      \softrst(\H_j) = \softrst(\H_{j'}).
    \end{equation}
    Further, since $j, j'$ are hard interventions,
    \begin{equation}
      \rst(\H_j) = \rst(\H_{j'}).
    \end{equation}
    In other terms, the high-level interventions must be distinct but their restrictions must be equal. When $j, j'$ are hard interventions this is not possible, thus $\omega$ is unique.
    \end{proof}
  \item $\omega(\emp) = \emp$,
    \begin{proof}
      Given an empty intervention~$\empty$, its soft restriction on the lower-level model contains all possible endogenous values. Formally,
      \begin{equation}
          \softrst(\M_\emp) = \val(\bm{X}).
      \end{equation}
      Since $\tau$ is surjective, it also holds that
      \begin{equation}
        \begin{split}
          \softrst(\H_j) &= \{\tau_{\bm{Y}}(\bm{x}) \mid \bm{x} \in \softrst(\L_\emp)\}\\
          &= \val(\bm{Y}).
        \end{split}
      \end{equation}
      In the same way, the empty restriction $\rst(\H_\emp)$ on the high-level model coincides with its domain. So, $\softrst(\H_\emp) = \val(\bm{Y})$ implies
      \begin{equation}
          \softrst(\H_\emp) = \{\tau_{\bm{Y}}(\bm{x}) \mid \bm{x} \in \softrst(\L_\emp)\}
      \end{equation}
      Therefore, by definition of low soft causal abstraction, it holds that $\omega(\emp) = \emp$.
    \end{proof}
  \item $i_1 \preceq i_2 \implies \omega(i_1) \preceq \omega(i_2)$,
\begin{proof}
  By ordering interventions according to their soft restriction, it follows that
  \begin{equation}
  \begin{split}
    i_1 \preceq i_2 &\implies \softrst(\L_{i_1}) \supseteq \softrst(\L_{i_2})\\
    &\implies \{\tau_{\bm{Y}}(\bm{x}) \mid \bm{x} \in \softrst(\L_{i_1})\} \supseteq \{\tau_{\bm{Y}}(\bm{x}) \mid \bm{x} \in \softrst(\L_{i_1})\}\\
    &\implies \softrst(\H_{\omega(i_1)}) \supseteq \softrst(\L_{\omega(i_2)})\\
    &\implies\omega(i_1) \preceq \omega(i_2).
  \end{split}
  \end{equation}
  \end{proof}
  \item There exists a subset of admissible hard interventions $\tilde{\I}$ such that $\H$ $\tau$-abstracts $\L$ given a subset of high-level interventions $\tilde{\J}$.
  \begin{proof}
    Let $\tilde{\I}$ be the subset of hard interventions in $\I$ and $\tilde{\J} = \{\omega(i)\mid i\in\tilde{\I}\}$. Then, since $\H$ low soft $\tau$-abstracts $\L$ for each intervention $i\in\I$ with there exists a partial function $\omega'\colon\tilde{\I}\to\tilde{\J}$ such that
    for any intervention~$i\in\tilde{I}$ it holds
    \begin{equation}
      \softrst(\H_{\omega(i)}) = \{\tau_{\bm{Y}}(\bm{x}) \mid \bm{x} \in \softrst(\L_i) \}
    \end{equation}
    and for every lower-level exogenous setting~$\bm{e}\in\val(\bm{E})$, it holds
    \begin{equation}
      \tau_{\bm{Y}}(\L_i(\bm{e})) = \H_{\omega(i)}(\tau_{\bm{U}}(\bm{e})).
    \end{equation}
    Thus, $\H$ $\tau$-abstracts $\L$ given the interventions sets $\tilde{\I}, \tilde{\J}$.
  \end{proof}
\end{enumerate}

\subsection{Ambiguity in Lower Soft Abstraction}\label{app:ambiguity}

\textbf{Theorem~\ref{theo:ambiguity}}:
Let $\H=\Htuple$ be an SCM with admissible interventions $\J$.
Given a variable ${V\in\bm{Y}}$, if $\J$ contains two distinct interventions $j ={(V \gets g)}$ and $j' ={(V \gets g')}$ s.t.
\begin{enumerate}
\item $\forall{\bm{u}\in\val(\bm{U})} \colon \proj(\H_j(\bm{u}), V) =
\proj(\H_{j'}(\bm{u}), V)$,
\item $\img(g) = \img(g')$,
\end{enumerate}
then, for any causal model~$\L$ with admissible interventions $\I$, whenever $\H$ low soft $\tau$-abstracts $\L$, the corresponding function $\omega\colon\I\to\J$ is not uniquely defined. 
\begin{proof}
    Since $\H$ lower soft $\tau$-abstracts $\L$, the function $\omega$ is surjective.
    Therefore, since $j_1, j_2\in\J$, there must exist two lower-level intervention $i_1, i_2$ such that
    $\omega(i_1) = j_1$ and $\omega(i_2) = j_2$.
    Therefore, it holds that
    \begin{align}
        \softrst(\M_{j_1}) &= \{\tau_{\bm{Y}}(\bm{x}) \mid \bm{x} \in \softrst(\L_{i_1})\}\\
        \softrst(M_{j_2}) &= \{\tau_{\bm{Y}}(\bm{x}) \mid \bm{x} \in \softrst(\L_{i_2})\}\\
        \tau_{\bm{Y}}(\L_{i_1}(\bm{e})) &= \H_{j_1}(\tau_{\bm{U}}(\bm{e}))\\
        \tau_{\bm{Y}}(\L_{i_2}(\bm{e})) &= \H_{j_2}(\tau_{\bm{U}}(\bm{e})).
    \end{align}
    Since the images $\img(g) = \img(g')$ of the two high-level function coincide, it also holds that
    \begin{equation}
        \softrst(\M_{j_1}) = \softrst(M_{j_2}),
    \end{equation}
    Furthermore, since $g, g'$ output the same values given every possible configuration of their parents, for every low-level exogenous setting $\bm{e}\in\val(\bm{E})$, it holds that
    \begin{equation}
       \bm{G}^{\H_{j_1}}(\tau_{\bm{U}}(\bm{e})) = 
       \bm{G}^{\H_{j_2}}(\tau_{\bm{U}}(\bm{e})).
    \end{equation}
    Consequently, while distinct, the two functions are equivalent given the constraints imposed by the definition of low soft abstraction. Therefore, there exist a function $\omega'$ such that
    \begin{equation}
        \omega'(i) = \begin{cases}
           \omega(i_1) & i = i_2\\ 
           \omega(i_2) & i = i_1\\ 
           \omega(i) & \text{otherwise}
           \end{cases},
    \end{equation}
    and $\H$ still lower soft $\tau$-abstracts $\L$.
    Consequently, $\omega$ is not unique.
\end{proof}

\subsection{Soft Abstraction}\label{app:softabstraction}

\textbf{Definition~\ref{def:softabstraction}}:
Let $\L=\Ltuple$ and $\H=\Htuple$ be two SCMs, $\I$ and $\J$ be their respective sets of admissible interventions, $\tau_{\bm{Y}} \colon \val(\bm{X}) \to \val(\bm{Y})$ be a surjective function between endogenous settings, and $\tau_{\bm{U}} \colon \val(\bm{E}) \to \val(\bm{U})$ be a surjective function between exogenous settings.
The model $\H$ is a soft $\tau$-abstraction of $\L$ whenever it exists a surjective function ${\omega\colon\I\to\J}$ such that, for any intervention~$i \in\I$ it holds
\begin{equation}
    \softrst(\H_{\omega(i)}) = \{\tau_{\bm{Y}}(\bm{x}) \mid \bm{x} \in \softrst(\L_i) \}
\end{equation}
and for every lower-level exogenous~$\bm{e}\in\val(\bm{E})$ and endogenous~$\bm{x}\in\val(\bm{X})$ setting, it holds
\begin{equation}
  \tau_{\bm{Y}}(\bm{F}^i(\bm{x}, \bm{e})) = 
  \bm{G}^{\omega(i)}(\tau_{\bm{Y}}(\bm{x}), \tau_{\bm{U}}(\bm{e})).
 \end{equation}

\noindent Given Definition~\ref{def:softabstraction} of Soft Abstraction, we prove that whenever $\H$ soft $\tau$-abstracts $\L$, it holds that
\begin{enumerate}
  \item $\omega$ is unique.
\begin{proof}
  We wish to show that for any distinct surjective function $\omega' \neq \omega$, $\H$ is not a soft $\tau$-abstraction of $\L$.
  Since $\omega$ and $\omega'$ are defined on $\I$, they are distinct if there exists at least one intervention $i\in\I$ such that $\omega(i) \neq \omega'(i)$.
  Thus, there must be two distinct high-level interventions $j={(\bm{V}\gets\bm{g})}$ and $j'={(\bm{W}\gets\bm{g'})}$ in $\J$ such that $\omega(i)=j$ and $\omega'(i)=j'$.
  Consequently, for every low-level exogenous setting~$\bm{e}\in\val(\bm{E})$ and endogenous setting~$\bm{x}\in\val(\bm{X})$, it must hold
  \begin{align}
    \tau_{\bm{Y}}(\bm{F}^i(\bm{x}, \bm{e})) &= 
    \bm{G}^{j}(\tau_{\bm{Y}}(\bm{x}), \tau_{\bm{U}}(\bm{e}))\\
    \tau_{\bm{Y}}(\bm{F}^i(\bm{x}, \bm{e})) &= 
    \bm{G}^{j'}(\tau_{\bm{Y}}(\bm{x}), \tau_{\bm{U}}(\bm{e})).
  \end{align}
  Therefore, it must also hold
  \begin{equation}
    G^{j}_Y(\tau_{\bm{Y}}(\bm{x}), \tau_{\bm{U}}(\bm{e})) =
    G^{j'}_Y(\tau_{\bm{Y}}(\bm{x}), \tau_{\bm{U}}(\bm{e})).
  \end{equation}
  for any high-level endogenous variable $Y\in\bm{Y}$, including those in $\bm{V}$, $\bm{W}$, or both.
  Consequently, since the output of the intervened functions match for any possible input, the high-level interventions $j, j'$ are not distinct and $\omega$ is unique.
\end{proof}
  \item $\omega(\emp) = \emp$.
  \begin{proof}
    Same proof to Point 2 in Appendix~\ref{app:lowersoftabstraction}.
  \end{proof}
  \item $i_1 \preceq i_2 \implies \omega(i_1) \preceq \omega(i_2)$.
  \begin{proof}
    Same proof to Point 3 in Appendix~\ref{app:lowersoftabstraction}.
  \end{proof}
  \item Given the subset of admissible high-level hard interventions $\tilde{\J}$, $\H$ low soft $\tau$-abstracts $\L$.
  \begin{proof}
    Since $\H$ soft $\tau$-abstracts $\L$, and $\tilde{\J}\subseteq\J$, for any $j\in\tilde{\J}$ there
    exists $i\in\I$ such that $\omega(i)=j$. Therefore, we define
    \begin{equation}
        \tilde{\I} = \{i\in\I\mid \omega(i)\in\tilde{\J} \}.
    \end{equation}
    Consequently, for any $i\in\tilde{\I}$ it holds that
    \begin{equation}
        \softrst(\H_{\omega(i)}) = \{\tau_{\bm{Y}}(\bm{x}) \mid \bm{x} \in \softrst(\L_i) \}
    \end{equation}
    and for every lower-level exogenous~$\bm{e}\in\val(\bm{E})$ and endogenous~$\bm{x}\in\val(\bm{X})$ setting, it holds
    \begin{equation}
      \tau_{\bm{Y}}(\bm{F}^i(\bm{x}, \bm{e})) = 
      \bm{G}^{\omega(i)}(\tau_{\bm{Y}}(\bm{x}), \tau_{\bm{U}}(\bm{e})).
    \end{equation}
    Therefore, it also holds
    \begin{equation}
      \tau_{\bm{Y}}(\L^i(\bm{e})) = 
      \H^{\omega(i)}(\tau_{\bm{U}}(\bm{e})).
    \end{equation}
    for any $\bm{e}\in\val(\bm{E})$.
    Thus, $\H$ low soft $\tau$-abstracts $\L$.
  \end{proof}
  \item Given the subset of admissible low and high-level hard interventions $\tilde{\I}, \tilde{\J}$, $\H$ $\tau$-abstracts $\L$.
  \begin{proof}
    The proof is equal to the previous,
    except we define the set of low-level admissible interventions as
    $\tilde{\I} = \{i \in \I \mid \omega(i)\in \J \land i \text{ is ``hard''} \land \omega(i) \text{ is ``hard}\}$,
    and consequently
    $\tilde{\J} = \{\omega(i) \mid i \in \tilde{\I}\}$.
  \end{proof}
\end{enumerate}

\subsection{Explicit Intervention Transformation}\label{app:explinttrans}

\textbf{Theorem~\ref{theo:explicitomega}}:
Let $\L=\Ltuple$ and $\H=\Htuple$ be two SCMs and $\I$ and $\J$ their corresponding sets of soft interventions.
Whenever $\H$ soft {$\tau$-abstracts} $\L$ with alignment $\Pi$,
$\omega$ maps the low-level intervention $i ={(\bm{V} \leftarrow \bm{f})}$
to an intervention $j ={(\bm{W} \leftarrow \bm{g})}$
on the high-level variables 
$\bm{W} \subseteq \{ Y \mid \bm{V} \cap \Pi(Y) \neq \emptyset \}$
such that
\begin{equation}
  g_Y(\bm{y}, \bm{u}) = 
  \tau_Y(F^i_{\Pi(Y)}(
  \tau^{-1}_{\bm{Y}}(\bm{y}),
  \tau^{-1}_{\bm{U}}(\bm{u}))),
\end{equation}
for any variable $Y\in\bm{W}$ and any partial inverse $\tau^{-1}$.

\begin{proof}
Since $\H$ soft $\tau$-abstracts $\L$, for each exogenous setting $\bm{e}\in\val(\bm{E})$, for each endogenous setting $\bm{x}\in\val(\bm{X})$, and each intervention $i\in\I$, it holds
\begin{equation}
    \tau_{\bm{Y}}(F^i(\bm{x}, \bm{e})) =
    G^{\omega(i)}(\tau_{\bm{X}}(\bm{x}), \tau_{\bm{E}}(\bm{e})).
\end{equation}
To simplify the notation, we denote as $\bm{x}\in\val(\bm{X}\cup\bm{E})$ any combination of exogenous and endogenous settings on the lower-level model.
Therefore, without loss of generality, the previous equation is equivalent to
\begin{equation}
    \tau_{\bm{Y}}(F^i(\bm{x})) =
    G^{\omega(i)}(\tau(\bm{x})).
\end{equation}
Given the mapping $\Pi$, the equation is also equivalent to
\begin{equation}
    \tau_Y(F^i_{\Pi(Y)}(\bm{x})) =
    G_Y^{\omega(i)}(\tau(\bm{x})).
\end{equation}
for any $Y\in\bm{Y}$. In the following, we prove that this equation is necessarily satisfied by our explicit definition of $\omega$ for any intervention~$i$ and any exogenous and endogenous configuration~$\bm{x}\in\val(\bm{X}\cup\bm{E})$.

Given any setting $\bm{x}$, applying a partial inverse $\tau^{-1}(\bm{x})$ returns
a possibly different value $\bm{x'}$ such that $\tau(\bm{x}) = \tau(\bm{x'})$.
Therefore, for any $\bm{x'}$ possibly different from $\bm{x}$, it holds that
\begin{equation}
    \begin{split}
    \tau(\bm{x})&=\tau(\bm{x'}) \\
    \implies \proj(\tau(\bm{x}),\pa(Y)) &=\proj(\tau(\bm{x'}),\pa(Y)) \\
    \implies G^{\omega(i)}_Y(\tau(\bm{x})) &=  G^{\omega(i)}_Y(\tau(\bm{x'}))
    \end{split}
\end{equation}
given that the structural equation $G^{\omega(i)}_Y$ depends only on parents $\pa(Y)$.
Consequently, whenever $\H$ soft $\tau$-abstracts $\L$,
\begin{equation}
    \begin{split}
    G^{\omega(i)}_Y(\tau(\bm{x})) &=  G^{\omega(i)}_Y(\tau(\bm{x'})) \\
    \implies \tau_Y(F^i_{\Pi(Y)}(\bm{x})) &= \tau_Y(F^i_{\Pi(Y)}(\bm{x'})).
    \end{split}
\end{equation}
Therefore, the necessary explicit form from Equation~\ref{eq:explicitomega}, results from
\begin{equation}
\begin{split}
    G_Y^{\omega(i)}(\tau(\bm{x})) &= 
    \tau_Y(F^i_{\Pi(Y)}(\bm{x}))\\
    &= \tau_Y(F^i_{\Pi(Y)}(\tau^{-1}(\tau(\bm{x}))))
\end{split}
\end{equation}
for any exogenous and endogenous configuration $\bm{x}\in\val(\bm{X}\cup\bm{E})$.
Since, $\tau$ is surjective,
for any high-level setting 
$\bm{y}\in\val(\bm{Y}\cup\bm{U})$,
there exist a low-level setting
$\bm{x}\in\val(\bm{X}\cup\bm{E})$
such that
$\tau(\bm{x}) = \bm{y}$, therefore
\begin{equation}
    G_Y^{\omega(i)}(\bm{y})
    = \tau_Y(F^i_{\Pi(Y)}(\tau^{-1}(\bm{y}))).
\end{equation}

Finally, we want to prove that intervening on a subset $\bm{V}$ of lower-level variables results \textit{at most} in an high-level intervention on the variables
$\bm{W} \subseteq \{ Y \mid \bm{V} \cap \Pi(Y) \neq \emptyset \}$
whose clusters intersect with $\bm{V}$. In fact, for any high-level variable $Y$ such that $\Pi(Y)\cap\bm{V}=\emptyset$, it holds that
\begin{equation}
\begin{split}
    G_Y^{\omega(i)}(\tau(\bm{x}))
    &=\tau_Y(F^i_{\Pi(Y)}(\bm{x}))\\
    &=\tau_Y(F_{\Pi(Y)}(\bm{x}))\\
    &= G_Y(\tau(\bm{x})).
\end{split}
\end{equation}
Therefore, applying $\omega$ results in the original structural equation $G_Y$ given an intervention $i$ not on in the cluster $\Pi(Y)$. Consequently, the variable $Y$ is not intervened by $\omega(i)$ and $Y\not\in\bm{W}$.
\end{proof}

\end{document}